\documentclass[10pt,letterpaper]{article}
\usepackage{spconf,amsmath,amssymb,graphicx,microtype,booktabs,cite}
\usepackage{etoolbox}
\usepackage[T1]{fontenc}
\usepackage{tikz}
\usetikzlibrary{arrows.meta,positioning,calc}
\definecolor{spmblue}{RGB}{39,84,138}
\definecolor{spmbluefill}{RGB}{232,240,249}
\definecolor{spmorange}{RGB}{210,102,38}
\definecolor{spmorangefill}{RGB}{252,239,228}
\definecolor{spmgreen}{RGB}{45,126,104}
\definecolor{spmgreenfill}{RGB}{229,244,239}
\definecolor{spmink}{RGB}{45,52,61}
\definecolor{spmgray}{RGB}{112,121,132}
\usepackage[hidelinks]{hyperref}
\microtypesetup{protrusion=false}
\makeatletter
\patchcmd{\@maketitle}{\vskip 2em}{\vskip 2.3em}{}{\errmessage{Title spacing patch failed}}
\makeatother
\hypersetup{pdftitle={OpenSPM: Transferable Spatial Pose Memory and Flow-Matching Control for Robotic Manipulation},pdfauthor={Iok Tong Lei, Qingchen Xie, Yifan Wang, Yap Ying Jie, Zhidong Deng}}

\title{\fontsize{14}{16}\selectfont OpenSPM: Transferable Spatial Pose Memory and Flow-Matching Control for Robotic Manipulation}
\name{Iok Tong Lei$^{1,2,*}$, Qingchen Xie$^{1,2,3,*}$, Yifan Wang$^{1,2,*}$, Yap Ying Jie$^{1,2}$, Zhidong Deng$^{1,2,\dagger}$\thanks{$^{*}$Equal contribution. $^{\dagger}$Corresponding author.}}
\address{$^{1}$Department of Computer Science and Technology, Tsinghua University\\
$^{2}$Institute for Artificial Intelligence, Tsinghua University\\
$^{3}$Department of Automation, Tsinghua University, Beijing 100084, China}
\begin{document}
\normalsize
\maketitle
\begin{abstract}
When object layouts change, a robot must preserve the geometry of a demonstrated interaction while correcting errors made during execution. OpenSPM separates these two problems. From one successful demonstration per task, it stores grasp- and placement-stage end-effector poses relative to their object anchors. Multi-view perception estimates the anchors in a new scene, and SE(3) transfer instantiates the stored poses. A 240K-parameter flow-matching expert generates local action chunks; measured-state feedback and terminal residual refinement keep the robot aligned with the transferred targets. On ten LIBERO-GOAL tasks (500 episodes), OpenSPM achieves 85.6\% success, 4.8389 ms chunk latency, and 1033.3 Hz equivalent action output frequency. Replacing memory with zero-shot key-pose prediction reduces success to 23.8\%; removing terminal refinement reduces it to 29.8\%, and open-loop execution reaches 3.0\%. The ablations identify both transferable contact geometry and execution feedback as necessary for the compact expert to succeed.
\end{abstract}

\section{Introduction}
A successful manipulation demonstration contains more than a sequence of actions. It specifies where contact should occur, from which direction the gripper should approach, and when grasping gives way to transport and release. These relations are tied to objects, whereas the recorded actions are tied to one scene. The central question of this paper is therefore: \emph{how can a robot preserve the interaction geometry of one demonstration, transfer it to a changed layout, and execute it reliably with a small action model?}

Existing approaches leave a gap between semantic generalization and geometric execution. Large vision-language-action (VLA) policies can map images and instructions directly to actions, but require substantial robot data and computation, and do not explicitly preserve demonstrated object-relative contacts. Replaying or imitating an entire trajectory overfits its world-frame layout. Predicting 6D contacts from a multimodal model at test time avoids replay but asks the model to rediscover physically valid grasp and placement poses. Finally, even a well-transferred target is insufficient when centimeter-level local prediction errors accumulate across action chunks. The problem thus has four coupled challenges: open-vocabulary 3D grounding, compact retention of contact geometry, cross-layout pose transfer, and feedback-stabilized action generation.

OpenSPM addresses them through four innovations. First, it combines semantic multi-view localization with complementary position and orientation estimation to form a consistent 6D object state. Second, it compresses one successful demonstration per task into phase-labeled, object-relative key spatial poses retrieved by language. Third, it re-anchors these poses through SE(3), preserving local interaction geometry under object displacement. Fourth, it uses a lightweight conditional flow-matching expert with measured-state conditioning and terminal residual refinement. Across ten LIBERO-GOAL tasks and 500 episodes, this decomposition achieves 85.6\% success with a 240K-parameter action expert; targeted ablations show that both spatial memory and feedback are indispensable.

\section{Related Work}
\textbf{Vision-language action policies.} RT-2, OpenVLA, Octo, and $\pi_0$~\cite{rt2,openvla,octo,pi0} provide broad semantic priors and direct visuomotor prediction. Their strength is end-to-end generality, but the action model must jointly learn semantics, geometry, and control from large datasets. OpenSPM instead exposes a low-dimensional geometric interface so that language selects experience while a compact model solves local motion.

\textbf{Semantic 3D perception.} Promptable segmentation and reconstruction models such as SAM 3 and SAM 3D~\cite{sam3,sam3d} extend recognition toward object geometry. Segmentation alone does not yield a metric 6D anchor, while single-view reconstruction can suffer from scale and coordinate-frame bias. Our perception module combines calibrated multi-view position constraints, reconstructed orientation, and state tracking; its purpose is not reconstruction for its own sake, but a common SE(3) interface for memory and control.

\textbf{Demonstration memory and spatial transfer.} Transporter Networks~\cite{transporter} exploit spatial structure in visual features for manipulation, while robomimic~\cite{robomimic} studies learning from offline human demonstrations. LIBERO~\cite{libero} evaluates task transfer. However, a stored image feature or world-frame action sequence does not directly specify how a demonstrated contact should move with an object. OpenSPM stores phase-labeled end-effector poses in the object frame, making experience language-addressable, geometrically transferable, and verifiable during execution.

\textbf{Generative action policies and feedback.} Diffusion Policy~\cite{diffusionpolicy} models multimodal action chunks by iterative denoising and uses receding-horizon control; flow matching~\cite{flowmatching} learns continuous transport fields. OpenSPM instead generates around transferred object-relative targets, conditions each chunk on the measured end-effector state, and refines terminal residuals at grasp and release. The distinction is cross-layout contact geometry and terminal-error suppression, not feedback in general.

\section{Methods}
Given instruction $l$, observations $I_t$, and measured end-effector pose $s_t$, OpenSPM retrieves and transfers an ordered sequence of spatial targets and generates operational-space actions $a_t=(\Delta p_t,\Delta\omega_t)\in\mathbb R^6$. Object and end-effector poses are homogeneous transforms in $SE(3)$. Figure~\ref{fig:framework} shows the demonstration and 3D pose-estimation inputs and the transfer of keyframe trajectories to a new layout.

\begin{figure*}[t]
\centering
\includegraphics[width=0.66\textwidth]{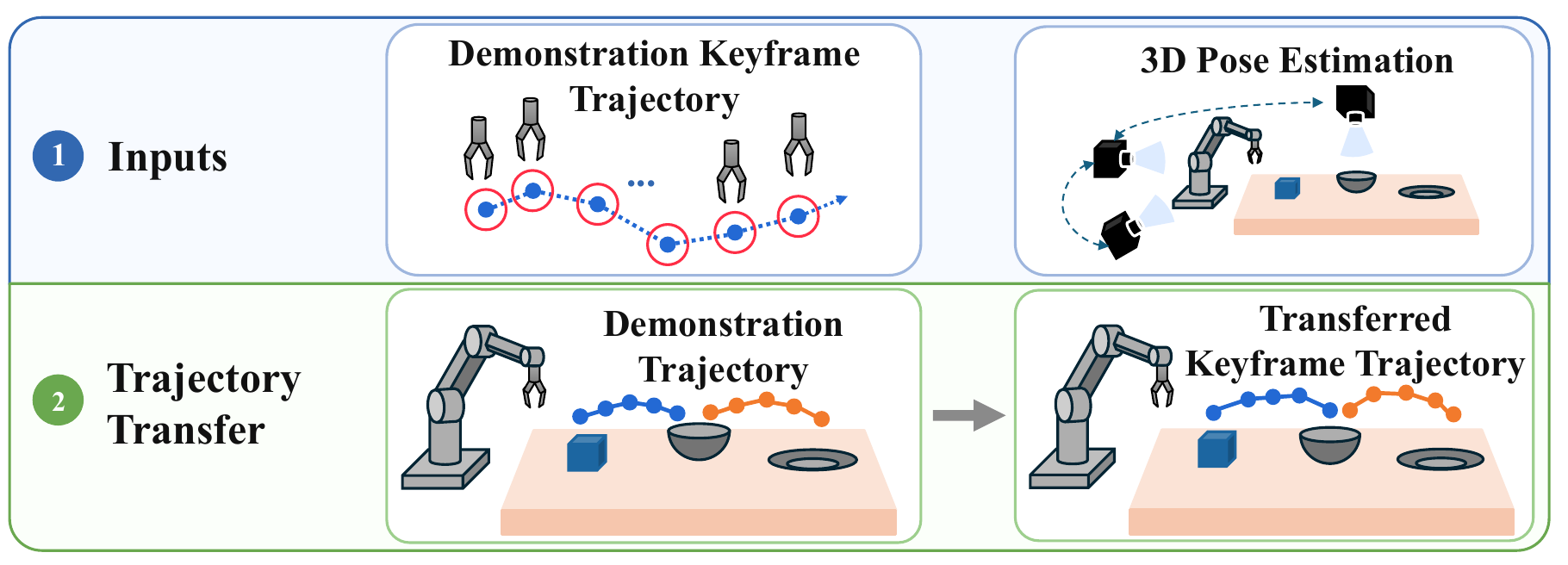}
\caption{OpenSPM trajectory-transfer pipeline; blue and orange denote grasping and placement.}
\label{fig:framework}
\end{figure*}

\subsection{Semantic 3D grounding and state estimation}
The first innovation converts language-grounded visual evidence into a stable object pose. For view $v$, the centroid of the semantic mask defines a calibrated world-frame ray with origin $o_v$ and unit direction $d_v$. Multi-view position is obtained by
\begin{equation}
\hat p=\arg\min_p\sum_v\|(I-d_vd_v^\top)(p-o_v)\|_2^2.
\end{equation}
Single-view 3D reconstruction supplies object orientation and shape, while the calibrated rays determine metric position. A state estimator then propagates pose between visual updates; after grasp, rigid object-to-end-effector kinematics supplies the state transition. The novelty is this complementary division: semantic perception identifies the object, multi-view geometry anchors it, reconstruction supplies orientation, and tracking maintains a control-ready 6D state.

\subsection{Spatial persistent memory and language retrieval}
The second innovation preserves only the task geometry that must survive a scene change. A successful demonstration is segmented into manipulation phases, and contact-relevant key poses are expressed relative to the active object anchor. A memory entry is
\begin{equation}
M_i=\left(l_i,\{(c_{i,j},\Delta T_{i,j}^{o\to e})\}_{j=1}^{K_i}\right),
\end{equation}
where $l_i$ is the language key, $c_{i,j}$ identifies the manipulation phase, and $K_i$ is the number of stored key poses. Let $o_{i,j}=\operatorname{Anchor}(l_i,c_{i,j})$ denote the phase-specific anchor: the manipulated object for grasp poses and the destination for placement poses. The stored transform is
\begin{equation}
\Delta T_{i,j}^{o\to e}=({}^{w}T_{o_{i,j}})^{-1}{}^{w}T_{e,i,j},
\label{eq:relative}
\end{equation}
where ${}^{w}T_{o_{i,j}}$ and ${}^{w}T_{e,i,j}$ are the corresponding demonstration anchor and end-effector poses in the world frame; $o$ in the superscript denotes the active anchor $o_{i,j}$. Unlike trajectory replay, the memory discards timing and world-frame redundancy while retaining approach direction, contact pose, release pose, and phase order. Language retrieves the relevant task entry; geometry determines whether and where it can be executed.

\subsection{Object-relative pose transfer}
The third innovation turns remembered geometry into scene-specific spatial targets. Given the current anchor pose ${}^{w}T'_{o_{i,j}}$, each end-effector target is instantiated as
\begin{equation}
{}^{w}T'_{e,i,j}={}^{w}T'_{o_{i,j}}\Delta T_{i,j}^{o\to e}.
\end{equation}
Grasp poses use the manipulated object as anchor, whereas placement poses use the destination. The transform preserves the demonstrated end-effector pose in the local object frame, so world-frame displacement changes the target location without changing the intended contact relationship. This explicitly separates \emph{what interaction should be preserved} from \emph{where it occurs}.

\subsection{Closed-loop flow-matching action generation}
The fourth innovation generates only the local motion between consecutive transferred targets and corrects it using physical state. Let $s_t$ denote the measured end-effector pose, $s_j$ the next target, $\delta(s_t,s_j)$ their relative pose, and $\tau$ segment progress. The conditional velocity field is
\begin{equation}
v_\theta(a_\xi,\xi\mid \phi(s_t,s_j,\delta,\tau)),
\end{equation}
and is trained along a Gaussian-to-demonstration interpolation with the standard flow-matching objective~\cite{flowmatching}
\begin{equation}
\mathcal L_{\rm FM}=\mathbb E\|v_\theta(a_\xi,\xi\mid c)-u_\xi\|_2^2.
\end{equation}
Here $c=\phi(s_t,s_j,\delta,\tau)$ is the conditioning representation, $a_\xi$ is the action chunk at flow time $\xi\in[0,1]$ along the Gaussian-to-demonstration interpolation, and $u_\xi$ is the corresponding target velocity along that interpolation.
Conditioning every chunk on the measured rather than nominal start state reduces inter-segment mismatch. At grasp and release targets, the same expert is reinvoked on the remaining residual until the translational tolerance is met or the maximum number of corrections is reached (Section~4.1). The method therefore uses one model for nominal motion and correction, without adding a second learned controller.

\section{Experiments}
\subsection{Implementation and evaluation protocol}
We evaluate all ten LIBERO-GOAL tasks~\cite{libero}, using the same 50 initial states per task for every variant (500 episodes). One successful trajectory per task populates the memory. Gripper-state transitions define grasp and placement phases; Qwen3-VL-Plus (API identifier \texttt{qwen3-vl-plus})~\cite{qwen35plus} selects keyframes from synchronized front, overhead, and side views by retaining gripper transitions, motion turns, and representative points on long segments. The same model localizes instructed objects, SAM 3 produces masks, and SAM 3D supplies front-view shape and orientation~\cite{sam3,sam3d}. Figure~\ref{fig:3view} shows the resulting perception stages. A Kalman predictor~\cite{kalman} maintains pose continuity, with rigid kinematics replacing reconstruction after grasp.

\begin{figure}[t]
\centering
\includegraphics[width=0.70\columnwidth]{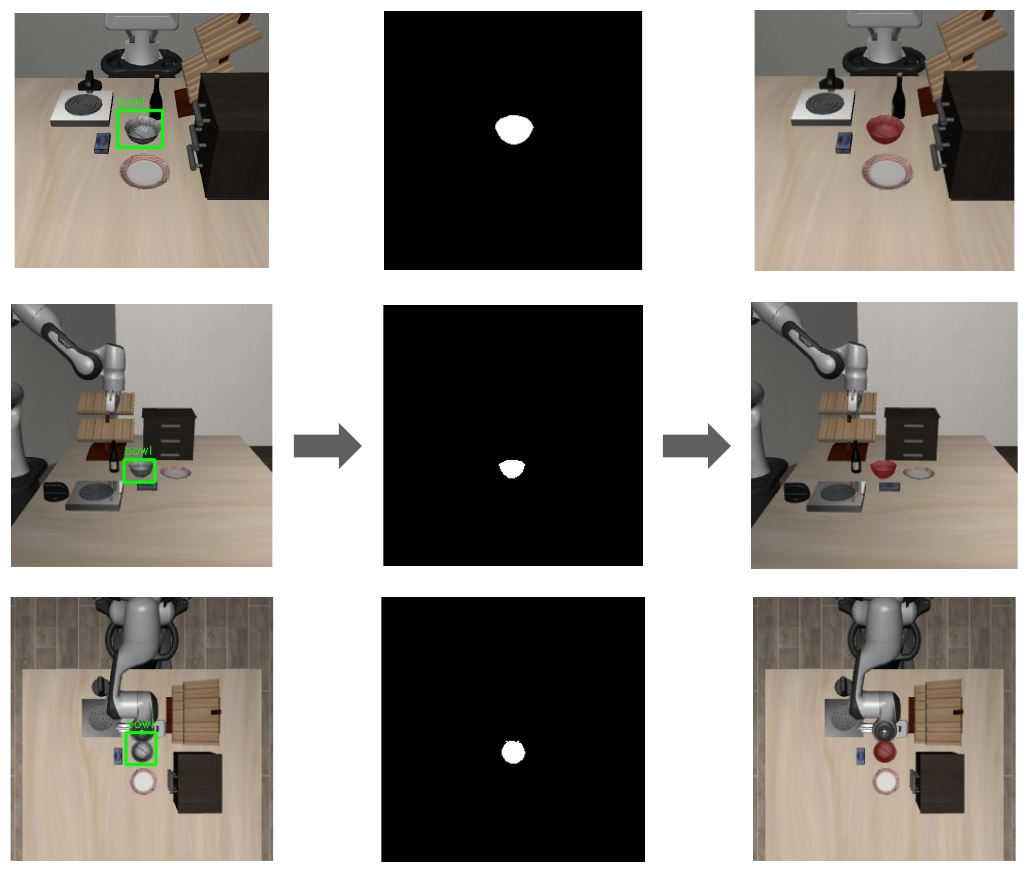}
\caption{Perception implementation: localized objects (left), semantic masks (center), and reconstructed overlays (right).}
\label{fig:3view}
\end{figure}

The shared action expert is trained on 4196 five-step chunks from 428 demonstration trajectories. Its condition concatenates 9D current and target poses (position plus continuous 6D rotation~\cite{zhou2019rotation}), their 6D relative pose, and normalized segment progress. The 240K-parameter condition encoder and step-wise decoder predict Cartesian translation and axis-angle increments; phase logic supplies the gripper command. Training uses Gaussian interpolation for flow matching. Inference integrates the learned field with ten Euler steps. Each chunk begins from the measured end-effector pose. At grasp and release, a translational residual exceeding 1 cm triggers another chunk, up to 30 corrections per target.

We report success rate (SR), chunk latency, equivalent action output frequency, average execution steps, and pose error. For chunk length $H$ and latency $t$ milliseconds,
\begin{equation}
f=1000H/t\quad\text{Hz}.
\end{equation}
This is model output throughput, not robot control rate; it excludes perception and environment stepping.

\subsection{Overall comparison}
Table~\ref{tab:main_results} compares representative diffusion, VLA, and world-action policies. Published baseline SRs use different training protocols, not a controlled rerun; parameters and $H$ follow the cited releases. The $\sim$260M Diffusion Policy entry is its 256M CNN diffusion network. Separate batch-1 timings use release-default precision and official LIBERO-GOAL checkpoints on one A100 80GB (Ubuntu 22.04/CUDA 12.1), averaged over 50 forward passes and excluding preprocessing and simulation. WorldVLA-512 is \texttt{model\_512/libero\_goal}; without a matched $H=5$ timing, $f$ is unreported rather than inherited from $H=25$. OpenSPM obtains 85.6\% SR with 4.8389 ms action-chunk latency and 1033.3 Hz equivalent output frequency.

\begin{table}[t]
\caption{LIBERO-GOAL comparison. $H$: chunk length; $f$: equivalent model output frequency, not robot control rate.}
\label{tab:main_results}
\centering
\setlength{\tabcolsep}{3pt}
\begin{tabular}{@{}lrrrr@{}}
\toprule
Method & Params. & $H$ & SR (\%) & $f$ (Hz)\\
\midrule
Diffusion Policy~\cite{openvla,diffusionpolicy,qu2025spatialvla} & $\sim$260M & 8 & 68.3 & 14.2\\
TraceVLA~\cite{heng2025tracevla} & 7B & 1 & 75.1 & 2.2\\
SpatialVLA~\cite{qu2025spatialvla} & 4B & 4 & 78.6 & 6.7\\
OpenVLA~\cite{openvla} & 7B & 1 & 79.2 & 4.9\\
WorldVLA-512~\cite{cen2025worldvla} & 7B & 5 & 83.4 & --\\
Octo-Base~\cite{openvla,octo,qu2025spatialvla} & 93M & 4 & 84.6 & 54.7\\
\midrule
OpenSPM & 0.24M & 5 & \textbf{85.6} & \textbf{1033.3}\\
\bottomrule
\end{tabular}
\end{table}

The SR is 1.0 percentage point above Octo-Base, while the action expert is orders of magnitude smaller than the VLA baselines. Because OpenSPM invokes external perception, parameters and latency characterize action generation rather than the whole system. Performance ranges from 44\% for bottle-on-cabinet to 100\% for bottle-on-rack; eight tasks reach at least 82\%.

\subsection{Perception and tracking accuracy}
Frame-by-frame SAM 3 + SAM 3D reconstruction~\cite{sam3,sam3d} provides offline visual references in the same base frame. Position error is Euclidean distance; orientation error is rotation distance in degrees. Table~\ref{tab:kalman} reports Kalman prediction error against those references, not absolute object-pose accuracy: the ten-task means are 5.8 mm/1.98$^\circ$.
Absolute reconstruction is evaluated separately on bowl-on-plate after center alignment. Over 50 layouts, bowl and plate position errors are 1.44 and 0.63 cm; the bowl height-axis error is 12.71$^\circ$. These values diagnose anchoring, not tracking.

\begin{table}[t]
\caption{Kalman pose MAE relative to offline visual references. Pos.: mm; Rot.: degrees. The last row gives arithmetic means.}
\label{tab:kalman}
\centering
\setlength{\tabcolsep}{7pt}
\begin{tabular}{@{}crrcrr@{}}
\toprule
Task & Pos. & Rot. & Task & Pos. & Rot.\\
\midrule
0 & 4.3 & 2.0 & 5 & 5.7 & 1.6\\
1 & 4.1 & 0.9 & 6 & 4.6 & 1.1\\
2 & 5.0 & 1.2 & 7 & 5.1 & 1.3\\
3 & 14.7 & 6.4 & 8 & 4.0 & 1.0\\
4 & 5.5 & 2.5 & 9 & 5.0 & 1.8\\
\midrule
\multicolumn{4}{l}{Ten-task arithmetic mean} & 5.8 & 1.98\\
\bottomrule
\end{tabular}
\end{table}

\subsection{Memory and closed-loop ablations}
Replacing MLLM-extracted demonstration key poses with zero-shot 6D key-pose prediction reduces SR from 85.6\% to 23.8\% (Table~\ref{tab:memory_ablation}). Perception, action expert, and evaluation protocol are unchanged. Spatial persistent memory therefore supplies demonstrated contact poses, approach directions, release positions, and phase order as executable geometric priors.
The 61.8-point loss is the central evidence that the memory is more than compression: the same low-level expert is much less effective when a multimodal model must infer physically usable key poses from the test image alone.

\begin{table}[t]
\caption{Key spatial pose memory ablation. Other modules and the evaluation protocol remain unchanged.}
\label{tab:memory_ablation}
\centering
\begin{tabular}{@{}lr@{}}
\toprule
Key-pose source & SR (\%)\\
\midrule
Zero-shot multimodal inference & 23.8\\
Spatial persistent memory (OpenSPM) & \textbf{85.6}\\
\bottomrule
\end{tabular}
\end{table}

Table~\ref{tab:closed_loop_ablation} separates open-loop concatenation, measured-state conditioning, and terminal residual refinement. Measured-state conditioning raises SR from 3.0\% to 29.8\%; terminal refinement raises it to 85.6\%, at the cost of more execution steps (278.5 versus 166.6). This overhead mainly arises from environment interaction and low-level control, with additional action-expert inference for corrective chunks. Local accuracy alone cannot ensure contact-rich execution.
On held-out chunks, the expert's mean errors are 1.92 cm in translation and 1.78$^\circ$ in rotation: sufficient to approach a transferred subgoal, but not to guarantee contact. Open-loop chunks accumulate planned-state mismatch; measured-state conditioning reduces that mismatch, and grasp/release refinement suppresses the remaining terminal offset. This explains the 3.0\% $\rightarrow$ 29.8\% $\rightarrow$ 85.6\% progression.

\begin{table}[t]
\caption{Closed-loop ablation. $\Delta$SR is the percentage-point difference from the full method.}
\label{tab:closed_loop_ablation}
\centering
\setlength{\tabcolsep}{5pt}
\begin{tabular}{@{}lrrr@{}}
\toprule
Execution & SR (\%) & $\Delta$SR & Avg. steps\\
\midrule
Open-loop & 3.0 & $-82.6$ & 173.2\\
Measured-state conditioning & 29.8 & $-55.8$ & 166.6\\
State feedback + refinement & \textbf{85.6} & 0 & 278.5\\
\bottomrule
\end{tabular}
\end{table}

\section{Discussion}
The evidence supports a division of labor: spatial memory preserves contact geometry, the small flow expert fills in local motion, and feedback ties that motion to physical state. Neither memory without correction nor fast local generation alone explains 85.6\% SR. OpenSPM trades action-model scale for an explicit geometric interface and external perception; it does not eliminate system complexity.

The limits follow the same chain. In the 44\% bottle-on-cabinet case, a biased reconstructed anchor can misplace the transferred target despite a sub-centimeter terminal residual. In plate pushing, an obstacle can invalidate the route, which local correction cannot replan. Feedback reduces drift but cannot repair a wrong target or infeasible path; robust grounding and contact-aware route adaptation remain necessary.

\section{Conclusion}
OpenSPM separates changed-layout manipulation into 3D grounding, object-centric interaction memory, pose transfer, and closed-loop flow matching. Its 85.6\% success over 500 LIBERO-GOAL episodes, together with the memory and feedback ablations, shows that contact geometry and execution correction are complementary. The 240K-parameter expert provides fast local generation; remaining failures point to perception and path feasibility as the next priorities.

\section{Acknowledgment}
No funding was received for this study. The authors declare that they have no relevant financial or nonfinancial interests or conflicts of interest.

\section{Compliance with Ethical Standards}
This study concerns robotic manipulation experiments and did not involve human participants or animals. Ethical approval was therefore not required.

\fontsize{9}{11}\selectfont
\bibliographystyle{IEEEbib}
\bibliography{references}

\begin{thebibliography}{10}

\bibitem{rt2}
Brianna Zitkovich et~al.,
\newblock ``{RT-2}: Vision-language-action models transfer web knowledge to
  robotic control,''
\newblock in {\em Proceedings of the Conference on Robot Learning (CoRL)},
  2023, vol. 229, pp. 2165--2183.

\bibitem{openvla}
Moo~Jin Kim et~al.,
\newblock ``{OpenVLA}: An open-source vision-language-action model,''
\newblock in {\em Proceedings of the Conference on Robot Learning (CoRL)},
  2025, vol. 270, pp. 2679--2713.

\bibitem{octo}
Dibya Ghosh, Homer~Rich Walke, Karl Pertsch, Kevin Black, Oier Mees, Sudeep
  Dasari, Joey Hejna, Tobias Kreiman, Charles Xu, Jianlan Luo, You~Liang Tan,
  Lawrence~Yunliang Chen, Quan Vuong, Ted Xiao, Pannag~R. Sanketi, Dorsa
  Sadigh, Chelsea Finn, and Sergey Levine,
\newblock ``Octo: An open-source generalist robot policy,''
\newblock in {\em Proceedings of Robotics: Science and Systems (RSS)}, 2024.

\bibitem{pi0}
Kevin Black et~al.,
\newblock ``{$\pi_0$}: A vision-language-action flow model for general robot
  control,''
\newblock in {\em Proceedings of Robotics: Science and Systems (RSS)}, 2025.

\bibitem{sam3}
Nicolas Carion, Laura Gustafson, Yuan-Ting Hu, et~al.,
\newblock ``{SAM 3}: Segment anything with concepts,''
\newblock {\em arXiv preprint arXiv:2511.16719}, 2025.

\bibitem{sam3d}
{SAM 3D Team} et~al.,
\newblock ``{SAM 3D}: {3Dfy} anything in images,''
\newblock {\em arXiv preprint arXiv:2511.16624}, 2025.

\bibitem{transporter}
Andy Zeng et~al.,
\newblock ``Transporter networks: Rearranging the visual world for robotic
  manipulation,''
\newblock in {\em Proceedings of the Conference on Robot Learning (CoRL)},
  2021, vol. 155, pp. 726--747.

\bibitem{robomimic}
Ajay Mandlekar et~al.,
\newblock ``What matters in learning from offline human demonstrations for
  robot manipulation,''
\newblock in {\em Proceedings of the Conference on Robot Learning (CoRL)},
  2022, vol. 164, pp. 1678--1690.

\bibitem{libero}
Bo~Liu, Yifeng Zhu, Chongkai Gao, Yihao Feng, Qiang Liu, Yuke Zhu, and Peter
  Stone,
\newblock ``{LIBERO}: Benchmarking knowledge transfer for lifelong robot
  learning,''
\newblock in {\em Advances in Neural Information Processing Systems}, 2023,
  vol.~36, pp. 44776--44791.

\bibitem{diffusionpolicy}
Cheng Chi et~al.,
\newblock ``Diffusion policy: Visuomotor policy learning via action
  diffusion,''
\newblock in {\em Proceedings of Robotics: Science and Systems (RSS)}, 2023.

\bibitem{flowmatching}
Yaron Lipman, Ricky T.~Q. Chen, Heli Ben-Hamu, Maximilian Nickel, and Matt Le,
\newblock ``Flow matching for generative modeling,''
\newblock in {\em Proceedings of the International Conference on Learning
  Representations (ICLR)}, 2023.

\bibitem{qwen35plus}
{Qwen Team},
\newblock ``{Qwen3-VL}: Sharper vision, deeper thought, broader action,'' Qwen
  API Platform, 2025,
\newblock Model identifier: \texttt{qwen3-vl-plus}. Accessed: Sep. 22, 2026.
  Available: \url{https://qwen.ai/apiplatform}.

\bibitem{kalman}
Rudolf~E. Kalman,
\newblock ``A new approach to linear filtering and prediction problems,''
\newblock {\em Journal of Basic Engineering}, vol. 82, no. 1, pp. 35--45, 1960.

\bibitem{zhou2019rotation}
Yi~Zhou, Connelly Barnes, Jingwan Lu, Jimei Yang, and Hao Li,
\newblock ``On the continuity of rotation representations in neural networks,''
\newblock in {\em Proceedings of the IEEE/CVF Conference on Computer Vision and
  Pattern Recognition (CVPR)}, 2019, pp. 5745--5753.

\bibitem{qu2025spatialvla}
Delin Qu, Haoming Song, Qizhi Chen, Yuanqi Yao, Xinyi Ye, Jiayuan Gu, Zhigang
  Wang, Yan Ding, Bin Zhao, Dong Wang, and Xuelong Li,
\newblock ``{SpatialVLA}: Exploring spatial representations for
  visual-language-action models,''
\newblock in {\em Proceedings of Robotics: Science and Systems (RSS)}, 2025.

\bibitem{heng2025tracevla}
Ruijie Zheng, Yongyuan Liang, Shuaiyi Huang, Jianfeng Gao, Hal {Daum{\'e} III},
  Andrey Kolobov, Furong Huang, and Jianwei Yang,
\newblock ``{TraceVLA}: Visual trace prompting enhances spatial-temporal
  awareness for generalist robotic policies,''
\newblock in {\em Proceedings of the International Conference on Learning
  Representations (ICLR)}, 2025.

\bibitem{cen2025worldvla}
Jun Cen et~al.,
\newblock ``{WorldVLA}: Towards autoregressive action world model,''
\newblock {\em arXiv preprint arXiv:2506.21539}, 2025.

\end{thebibliography}
\end{document}